\documentclass[conference]{IEEEtran}
\IEEEoverridecommandlockouts
\usepackage{cite}
\usepackage{textcomp}
\def\BibTeX{{\rm B\kern-.05em{\sc i\kern-.025em b}\kern-.08em
    T\kern-.1667em\lower.7ex\hbox{E}\kern-.125emX}}

\usepackage{xcolor}
\usepackage{hyperref}       
\usepackage{url}            
\usepackage{booktabs}       
\usepackage{amsfonts}       
\usepackage{amsmath}
\usepackage{amssymb}
\usepackage{nicefrac}       
\usepackage{microtype}      
\usepackage{lipsum}
\usepackage{fancyhdr}       
\usepackage{graphicx}       
\graphicspath{{media/}}     
\usepackage{wrapfig}
\usepackage{multirow}
\usepackage{array}
\usepackage{subcaption}
\usepackage[ruled, lined, linesnumbered, longend]{algorithm2e}
\SetKw{Break}{break}

\DeclareMathOperator*{\argmax}{\arg\!\max}
\usepackage{amsthm}
\newtheorem{theorem}{Theorem}
\newtheorem{lemma}{Lemma}
\newtheorem{definition}{Definition}

\usepackage[capitalize]{cleveref}
\crefname{table}{Table}{Tabs.}
\crefname{figure}{Fig.}{Figs.}
\Crefname{section}{Section}{Sections}
\Crefname{table}{Table}{Tables}
\Crefname{assumption}{Assumption}{Assumptions}
\crefname{algorithm}{Algorithm}{Algoritms}

\begin{document}

\title{Local Differential Privacy in Graph Neural Networks: a Reconstruction Approach}

\author{\IEEEauthorblockN{Karuna Bhaila}
\IEEEauthorblockA{\textit{University of Arkansas}\\
Fayetteville, USA \\
kbhaila@uark.edu}
\and
\IEEEauthorblockN{Wen Huang}
\IEEEauthorblockA{\textit{University of Arkansas}\\
Fayetteville, USA \\
wenhuang@uark.edu}
\and
\IEEEauthorblockN{Yongkai Wu}
\IEEEauthorblockA{\textit{Clemson University}\\
Clemson, USA \\
yongkaw@clemson.edu}
\and
\IEEEauthorblockN{Xintao Wu}
\IEEEauthorblockA{\textit{University of Arkansas}\\
Fayetteville, USA \\
xintaowu@uark.edu}
}

\maketitle              
\begin{abstract}
Graph Neural Networks have achieved tremendous success in modeling complex graph data in a variety of applications. However, there are limited studies investigating privacy protection in GNNs. In this work, we propose a learning framework that can provide node privacy at the user level, while incurring low utility loss. We focus on a decentralized notion of Differential Privacy, namely Local Differential Privacy, and apply randomization mechanisms to perturb both feature and label data at the node level before the data is collected by a central server for model training. Specifically, we investigate the application of randomization mechanisms in high-dimensional feature settings and propose an LDP protocol with strict privacy guarantees. Based on frequency estimation in statistical analysis of randomized data, we develop reconstruction methods to approximate features and labels from perturbed data. We also formulate this learning framework to utilize frequency estimates of graph clusters to supervise the training procedure at a sub-graph level. Extensive experiments on real-world and semi-synthetic datasets demonstrate the validity of our proposed model.
\end{abstract}


\section{Introduction}\label{sec:intro}
Graph data are ubiquitous in the modern world allowing graph-structured representation for complex data in the realm of social networks, finance, biology, and so on. 
Graph Neural Networks (GNNs) have been widely adopted in such domains to model the expressive nature of graph-structured data~\cite{4700287,DBLP:journals/tnn/WuPCLZY21}. GNNs rely on \textit{message-passing} mechanisms to propagate information between graph nodes and output embeddings that encode both node features and neighborhood features aggregated using graph adjacency information. These embeddings are used in predictive downstream tasks for meaningful applications such as drug discovery~\cite{DBLP:journals/bioinformatics/ZitnikAL18}, traffic prediction~\cite{DBLP:conf/nips/0001YL0020}, recommendation~\cite{DBLP:conf/wsdm/ChenW21}. 
This widespread prevalence of GNNs, however, raises concerns regarding the privacy of sensitive information whose leakage may lead to undesirable and even harmful consequences. GNNs have been shown to be vulnerable to several privacy risks including membership inference~\cite{DBLP:conf/tpsisa/OlatunjiNK21}, link re-identification~\cite{DBLP:conf/uss/HeJ0G021}, and attribute disclosure~\cite{DBLP:conf/kdd/ZhelevaG07}. This risk of data leakage is considerably higher in GNNs compared to traditional learning models due to the presence of additional graph structure information~\cite{DBLP:conf/tpsisa/OlatunjiNK21}. To ensure compliance with legal data protection guidelines~\cite{nist-privacy-framework} and for the protection of user privacy, GNNs must thus be trained and deployed in a privacy-preserving manner. 

In this paper, we aim to address such privacy concerns in GNNs. We focus on a specific scenario of node privacy wherein node-level features and labels are held locally by each user and the global graph structure is available to the server that hosts applications. The server could benefit from users' feature data which paired with graph topology can be utilized for embedding generation and/or predictive modeling via GNNs. However, collecting user feature and label data, possibly containing sensitive and identifying information, may incur serious privacy issues. To this end, Local Differential Privacy (LDP)~\cite{DBLP:conf/focs/KasiviswanathanLNRS08} is often adopted in data collection for training machine learning or releasing statistics in a private manner~\cite{DBLP:journals/iotj/ChamikaraBKLCA20}. Furthermore, it has been deployed in large-scale data-gathering of user behavior and usage statistics at Apple~\cite{apple-dp} and Google~\cite{DBLP:conf/ccs/ErlingssonPK14} motivating the integration of LDP in data collection for GNNs as well. 

\noindent\textbf{Challenges} 
The main challenge in training GNNs with privately collected data arises due to the utility-privacy trade-off of differentially private mechanisms. With randomization of data at an individual level, privatized data oftentimes misrepresents the data distribution of the population. A learning model that learns feature and label correlation from this data overfits the noise and is unable to achieve good utility on predictive and analytical tasks with unseen data. Furthermore, since GNNs propagate information throughout the graph to output node embeddings, the quality of the embeddings also suffers due to additive noise present in the training data after applying LDP mechanisms. 

\noindent\textbf{Prior Work}
A few recent works have attempted to address node privacy in GNNs~\cite{DBLP:journals/corr/abs-2111-15521,DBLP:journals/corr/abs-2109-08907} but they enforce privacy only during training and/or model release which puts user information at risk if the server is malicious. Most notably, Sajadmanesh and Gatica-Perez~\cite{DBLP:conf/ccs/SajadmaneshG21} propose a node-level LDP framework in the distributed setting where features and labels are held private by the user, and the graph structure is known to the server. They propose an LDP protocol called multi-bit mechanism to perturb node features by extending the 1-bit mechanism~\cite{DBLP:conf/nips/DingKY17} to multidimensional features. The multi-bit mechanism randomly selects a subset of features for each user, transforms each selected feature value to either 1 or -1, 
and indiscriminately reports the value 0 for the remaining ones. To protect label privacy, node labels are perturbed using Randomized Response (RR)~\cite{10.2307/2283137}. A GCN-based multi-hop aggregator is then prepended to the GNN model for implicit denoising of both features and labels. They further implement forward loss correction~\cite{DBLP:conf/cvpr/PatriniRMNQ17} to supervise the learning process in the presence of noisy labels. However, the multi-bit mechanism potentially results in a huge loss of information, especially considering that the size of the sampled feature subset is set to 1 according to the analysis presented in the paper. 
The model is evaluated on several graph datasets with high-dimensional binary features where each feature has around $99\%$ zero values, as we show in \cref{sec:datasets}.  
This inadvertently aids to reduce variance during aggregation of features perturbed via the multi-bit mechanism. This may not be the case during deployment in the real world which could significantly affect the privacy-utility trade-off of the model.

\noindent\textbf{Contributions} 
To address the aforementioned challenges, in this work, we propose RGNN, a novel reconstruction-based GNN learning framework that can guarantee node privacy while incurring low utility loss. The focus in this work is on LDP and we randomize both feature and label data at the user level. To protect feature privacy, we extend previous work by Arcolezi et al.~\cite{DBLP:conf/cikm/ArcoleziCBX21} and implement Generalized Randomized Response with Feature Sampling (GRR-FS), a randomization framework with provable LDP guarantees. We theoretically derive frequency estimation and its variance for variables randomized with GRR-FS; these perturbed features and labels are collected by an aggregator as training data. To minimize utility loss caused by randomization, we propose reconstruction methods that approximate true features and label distributions from the perturbed data via frequency estimation. We use these methods to estimate data distributions at a sub-graph level and ultimately at the node level. We further introduce propagation during reconstruction to reduce estimation variance. We also formulate this learning framework to utilize frequency estimates of graph clusters to supervise the training procedure. 
The proposed framework can be used with any GNN architecture and does not require private data for training or validation. We perform extensive experiments on four real-world and two semi-synthetic datasets for the task of transductive node classification under varying privacy budgets. Empirical results show our method's effectiveness and superiority over baselines. 

\section{Related Work}
Privacy leakage in GNNs has become an unavoidable concern due to real-world implications of models trained on potentially sensitive data. 
In order to prevent privacy leakage and protect against malicious privacy attacks, various attempts have been made to develop privacy-preserving GNN algorithms, including the extension of Differential Privacy (DP) to GNNs. 
Wu et al.~\cite{DBLP:conf/sp/0011L0022} propose an edge-level DP algorithm by adding noise to the adjacency matrix as a pre-processing step. However, the edge-DP problem is different from the node privacy setting explored in this paper where the graph topology is non-private but the node features and labels are locally private. Zhang et al.~\cite{DBLP:journals/corr/abs-2210-04442} propose a node-level DP algorithm that decouples message-passing from feature aggregation and uses an approximate personalized PageRank to perform feature transformation instead of message-passing. This algorithm, however, requires a trusted aggregator to compute the approximate personalized PageRank matrix from private data and is more suited for private release of trained GNN models and their outputs than ensuring user-level privacy. Daigavane et al.~\cite{DBLP:journals/corr/abs-2111-15521} propose a node-level DP algorithm by integrating sub-graph sampling and the standard DP-SGD~\cite{DBLP:conf/ccs/AbadiCGMMT016} algorithm into the training framework to update GNN parameters, but this approach also requires a trusted aggregator to perform sampling before model training and is limited to a 1-layer GNN model. Olatunji et al.~\cite{DBLP:journals/corr/abs-2109-08907} utilize the student-teacher training workflow from the PATE framework~\cite{DBLP:conf/iclr/PapernotSMRTE18} to release a student GNN trained on public data and partly private labels obtained via the teacher GNN with DP guarantees. However, their framework necessitates the availability of public graph data which is not possible under LDP constraints considered in this work. Tran et al.~\cite{DBLP:conf/bigdataconf/TranLPKMKTW22} propose a randomization mechanism to ensure node-level and edge-level DP by optimizing the magnitude of noise injected into nodes’ features and the graph structure. However, node labels are considered non-private in their setting. Our work is most closely related to that of Sajadmanesh and Gatica-Perez~\cite{DBLP:conf/ccs/SajadmaneshG21} where node features and labels are individually perturbed before sending to a server for training. They utilize GCN-based multi-hop aggregation for denoising features, and labels and incorporate forward loss correction~\cite{DBLP:conf/cvpr/PatriniRMNQ17} to facilitate training with noisy labels. However, implicit denoising via aggregation may not be adequate to obtain accurate graph signals due to propagation of noise during aggregation. 

\section{Preliminaries}
We use $\mathcal{G} = (\mathcal{V, E})$ to denote an unweighted and undirected graph where $\mathcal{V}$ is a set of $N$ nodes and $\mathcal{E} \subseteq \mathcal{V} \times \mathcal{V}$ defines graph edges represented by an adjacency matrix $\mathbf{A} \in \{0,1\}^{N\times N}$ such that $\mathbf{A}_{u,v}=1$ if an edge exists between nodes $u$ and $v$. $\mathbf{X} = \{\mathbf{x}_1,\hdots\mathbf{x}_N\}$ represents node features, where $\mathbf{x}_v \in \mathbb{R}^d$ is the $d$-dimensional feature vector of node $v$ and we use $\mathbf{X}_i \in \mathbb{R}^N$ to denote the $i$-th feature column. In a transductive setting, a fraction of nodes denoted $\mathcal{V}^L$ are provided with labels. For each node $v \in \mathcal{V}^L$, a label vector $\mathbf{y}_v \in \{0,1\}^c$ s.t. $\sum \mathbf{y}_v = 1$ indicates its class membership. $\mathcal{V}^U = \mathcal{V} - \mathcal{V}^L$ is the set of unlabeled nodes whose labels are to be predicted. Further, we use $\mathcal{C}_r \subseteq \mathcal{V}$ to denote an arbitrary cluster of $\mathcal{G}$ containing a subset of nodes and we use $\mathbf{b}_r$ to refer to the label distribution of $\mathcal{C}_r$ which is defined as the average of labels $\mathbf{y}_v$ of all $v \in \mathcal{C}_r$. Finally $\mathbf{Y}\in\{0,1\}^{N\times c}$ is the node label matrix where $\mathbf{y}_v$ is an all-$0$ vector for any $v \in \mathcal{V}^U$, and $\mathbf{B}\in\mathbb{R}^{C\times c}$ denotes the label distribution for $C$ clusters. 

\begin{figure*}[htb]
    \centering
    \includegraphics[width=0.75\textwidth]{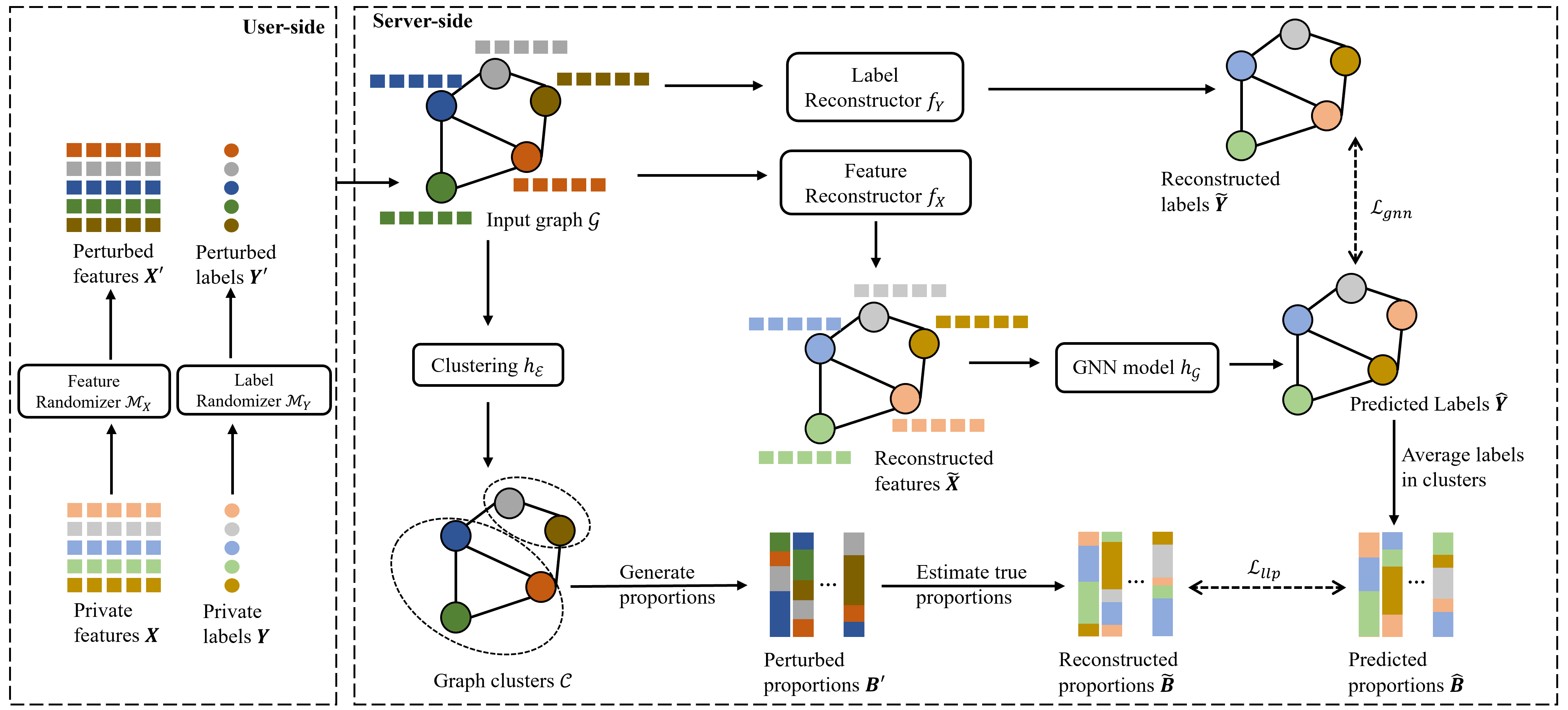}
    \caption{An overview of the proposed framework RGNN}
    \label{fig:model}
\end{figure*}

\subsection{Graph Neural Networks}
GNNs aim to learn node representations for nodes in a given graph by utilizing node attributes and edges. The learned representations are useful for downstream tasks such as node classification, link prediction, and graph classification. A $k$-layer GNN consists of $k$ sequential graph convolutional layers which implement \textit{message passing} mechanisms to update representations for nodes using aggregated representations from their neighbor nodes. The updating process of the $k$-th layer in GNN is generally formulated as
\begin{equation}
\begin{split}
    \mathbf{h}^k_{\mathcal{N}(v)} & = \operatorname{AGGREGATE}^{k}(\{\mathbf{h}^{k-1}_{u},\:\forall u \in \mathcal{N}(v)\}),\\
    \mathbf{h}^k_v & = \operatorname{UPDATE}^{k}(\mathbf{h}^k_{\mathcal{N}(v)};\mathbf{W}^k),
\end{split}
\end{equation}
where $\mathbf{h}^k_v$ is node $v$'s representation vector at the $k$-th layer, $\mathcal{N}(v)$ is its neighborhood set, and $\mathbf{W}^k$ defines the parameters of the learnable $\textnormal{UPDATE}$ function. GCN~\cite{DBLP:conf/iclr/KipfW17}, GraphSAGE~\cite{DBLP:conf/nips/HamiltonYL17}, and GAT\cite{DBLP:conf/iclr/VelickovicCCRLB18} are some of the most widely used GNNs.

\subsection{Differential Privacy} \label{sec:dp}
Since first proposed by Dwork et al.~\cite{DBLP:conf/tcc/DworkMNS06}, differential privacy has been established as the de-facto definition for privacy guarantees. In its seminal work, DP provided privacy guarantees for databases in the centralized setting where a trusted curator holds a database containing users' private information and answers queries about the database. 
In this work, we focus on a decentralized definition of DP referred to as Local Differential Privacy~\cite{DBLP:conf/focs/KasiviswanathanLNRS08}. As opposed to DP which requires a trusted aggregator, LDP provides privacy guarantees on the user side by letting each user perturb their own private data before sending it to an aggregator. This way, the aggregator cannot access the true and private data of any user alleviating the need for a trusted aggregator.

\begin{definition}\label{def:LDP}
$\epsilon$-LDP~\cite{DBLP:conf/focs/KasiviswanathanLNRS08}. Given $\epsilon > 0$, a randomized mechanism $\mathcal{M}$ satisfies $\epsilon$-local differential privacy, if for any pairs of user's private data $x$ and $x'$, and for any possible outputs $o \in Range(\mathcal{M})$, we have
$\textnormal{Pr}[\mathcal{M}(x) = o] \leq e^{\epsilon} \cdot \textnormal{Pr}[\mathcal{M}(x') = o]$, where $\epsilon$ is the privacy budget that controls the trade-off between utility and privacy of $\mathcal{M}$.
\end{definition}
Essentially, LDP ensures that an adversary is unable to infer the input values of any target individual using the output values obtained. To achieve LDP in data statistics and analysis, mechanisms such as randomized response, histogram encoding, unary encoding, or local hashing are applied during the collection of user data that are categorical in nature~\cite{DBLP:journals/sensors/WangZFY20}.

\subsection{Sampling for Privacy Amplification}\label{sec:DP_sampling}
A widely accepted approach to strengthen privacy guarantees of DP mechanisms is to take advantage of the aggregator/adversary's uncertainty by adding a sampling step before privatization~\cite{DBLP:conf/ccs/LiQS12, DBLP:conf/cikm/ArcoleziCBX21}. Sampling exploits such uncertainty and allows for the relaxation of privacy constraints in a randomization mechanism while ensuring strict privacy guarantees.
\begin{lemma}\label{lem:amplify} Amplification Effect of Sampling~\cite{DBLP:conf/ccs/LiQS12}. 
    Let $\mathcal{M}$ denote an algorithm that guarantees $\epsilon'$-DP over some data. Also, let $\mathcal{M}^{\beta}$ denote an algorithm that first samples tuples from the data with probability $\beta$ and then applies $\mathcal{M}$ on the sampled data. Then, $\mathcal{M}^{\beta}$ satisfies DP with $\epsilon = \ln\left(1+\beta(e^{\epsilon'}-1)\right)$.
\end{lemma}

\subsection{Learning from Label Proportions}\label{sec:llp}
Learning from label proportions (LLP) is an alternative supervision method introduced in~\cite{DBLP:journals/jmlr/QuadriantoSCL09} to train predictive models when instance labels are too difficult or expensive to obtain. 
In this setting, instances are grouped into iid bags and only the label distributions of these bags are known to the learning model. The goal is then to train an instance label predictor by training the learner on the bag-level aggregate information. 

Denote by $\mathcal{B}_r$, an arbitrary bag containing instances with corresponding labels. The learner cannot access the instance labels $\mathbf{y}$ and only receives bag proportions $\mathbf{b}_r$ computed as an average over the instance labels in bag $\mathcal{B}_r$.
Then an instance label predictor can simply compute its prediction of the bag proportion $\mathbf{\hat{b}}_r$ by averaging over the predicted labels $\mathbf{\hat{y}}$ of instances in bag $\mathcal{B}_r$. Unlike a traditional learning model, this LLP-based learner calculates training loss between the true and predicted distributions of the bags. For a bag $\mathcal{B}_r$, this proportion loss can be computed as the KL divergence between true and predicted proportions as
$D_{KL}(\mathbf{\hat{b}}_r||\mathbf{b}_r)$.
The objective of the learner is to minimize $D_{KL}(\cdot)$ for all bags in the training dataset. It has been previously shown that minimizing bag proportion prediction error guarantees a good instance label predictor under the assumption that the labels are not evenly distributed in all bags~\cite{yu2015learning}.

\section{Reconstruction-based Private GNN}
We formally define the problem of node-level LDP in GNNs. 
Let $\mathcal{M}$ denote some randomization mechanism that provides $\epsilon$-LDP. Then, $\mathbf{X}' = \mathcal{M}(\mathbf{X})$ and $\mathbf{Y}' = \mathcal{M}(\mathbf{Y})$ refer to user feature and label information collected using $\mathcal{M}$ to ensure user privacy. Also, let $h(\cdot)$ define a GNN model for node classification that takes node features and adjacency matrix as the input and outputs label predictions. In this decentralized differential privacy setting, the server responsible for model training has access to the non-private adjacency $\mathbf{A}$, randomized features $\mathbf{X'}$ and randomized labels $\mathbf{Y'}$ but not the true private features $\mathbf{X}$ and labels $\mathbf{Y}$. In this scenario, we aim to train a GNN on $\mathbf{X}'$ and $\mathbf{Y}'$ to learn a mapping $h(\mathbf{X}',\mathbf{A};\mathbf{W})\rightarrow \mathbf{\hat{Y}}$ that can estimate accurate labels for $v \in \mathcal{V}^U$.

\cref{fig:model} illustrates the overall framework of RGNN which is composed of two perturbation models: $\mathcal{M}_X$ and $\mathcal{M}_Y$ implemented at the node level to inject noise into the feature and label data respectively. 
On the server side, the feature reconstructor $f_X$ takes the non-private graph structure $\mathcal{G}$ and the perturbed features $\mathbf{X}'$ as input to derive estimated features $\mathbf{\Tilde{X}}$ via propagation and frequency estimation for all nodes in the graph. The label reconstructor $f_Y$ also takes the graph structure $\mathcal{G}$ and perturbed labels $\mathbf{Y}'$ as input to reconstruct labels for the labeled set of nodes $\mathcal{V}^L$. Additionally, an edge-based graph clustering algorithm $h_\mathcal{E}$ partitions the graph into a set of clusters $\mathbf{\mathcal{C}}$. For each cluster $\mathcal{C}_r \in \mathbf{\mathcal{C}}$, a cluster-level label proportion denoted as $\mathbf{b}_r'$ is computed by aggregating over the nodes in $\mathcal{C}_r$. By means of frequency estimation, an accurate cluster-level proportion $\mathbf{\Tilde{b}}_r$ is derived from the perturbed cluster label proportion. A GNN model $h_\mathcal{G}$ is then trained on the reconstructed features and labels with additional supervision provided by the reconstructed cluster proportions. Next, we discuss each component in more detail.        

\subsection{Reconstruction with Private Features}\label{sec:f_X}
LDP frequency oracles used to obtain private data statistics can also be adopted during training data collection to provide privacy protection to users via plausible deniability. However, in contrast to the one-dimensional data collected for such purpose, data for model training are usually multi-dimensional with mutual dependencies between features. Ensuring formal privacy guarantees then requires users to report privatized values for each feature. In other words, the total privacy budget gets compounded along the feature dimension and results in a high noise rate often at the cost of model performance~\cite{DBLP:journals/sensors/WangZFY20}. To circumvent this issue, sampling and noisy data generation techniques have been adopted in~\cite{DBLP:conf/ccs/SajadmaneshG21, DBLP:conf/cikm/ArcoleziCBX21} where each user applies randomization techniques only on a few sampled features and reports some default or noisy values for the non-sampled features. The RS+FD method proposed in~\cite{DBLP:conf/cikm/ArcoleziCBX21} samples one of $d$ total features and applies an LDP protocol on it. For the remaining $d - 1$ features, the method reports completely random values drawn from a uniform distribution over each feature domain. On the other hand, the multi-bit mechanism in~\cite{DBLP:conf/ccs/SajadmaneshG21} randomizes one of $d$ features but reports a default value of 0 for the rest. As discussed prior, the multi-bit mechanism may require high feature matrix sparsity to achieve good utility. 

Therefore, in this paper, we focus on sampling and fake data generation using a uniform distribution. Still, in high-dimensional feature settings, reporting $d - 1$ features from a uniform distribution may result in highly noisy data unsuitable for training. So, we extend this method to a more general case such that each user samples $m$ of $d$ features and randomizes them independently. Consequently, the privacy budget $\epsilon$ is compounded over $m$ sampled features, which is more favorable than compounding over $d$ features. 
On the client side, a user samples $m$ features then implements the Generalized Randomized Response method~\cite{DBLP:conf/edbt/0009WH16} on the sampled features, i.e.: the user reports their true value with probability $p$ and reports any other value with probability $q$. Here, $p$ and $q$ are set as $\frac{e^{\epsilon_x}}{e^{\epsilon_x}+\gamma_i-1}$ and $\frac{1}{e^{\epsilon_x}+\gamma_i-1}$ respectively where $\gamma_i$ indicates the domain size of the sampled feature $x$ and $\epsilon_x$ is the allocated privacy budget for feature $x$. This probabilistic transition can also be represented using a matrix \mbox{$\mathbf{P} \in \mathbb{R}^{\gamma_i\times\gamma_i}$} whose diagonal entries equal $p$ and the non-diagonal entries equal $q$. For $d-m$ features, the user reports a random value drawn from a uniform distribution over each feature. We call this method Generalized Randomized Response with Feature Sampling (GRR-FS). 

\begin{theorem}\label{thm:GRR-FS_privacy_budget}
    GRR-FS satisfies $\epsilon_X$-LDP where $\epsilon_X = \ln{\left(1 + \frac{m}{d}(e^{m\epsilon_x}-1)\right)}$.
\end{theorem}

\begin{proof}
    Consider a feature vector $\mathbf{x} = \{x_1,\hdots,x_d\}$ on which GRR-FS is applied with feature sampling probability $\beta = m/d$ and privacy budget $\epsilon_x$ on each sampled feature. Let $\gamma_i$ refer to the domain size of feature $x_i$ s.t. $x_i \in \{1,\hdots,\gamma_i\}$. The private feature vector $\mathbf{x}'=\{x'_1,\hdots,x'_d\}$ is obtained as:
    \begin{align}
        x'_i = \left\{ \begin{array}{rl}
                \operatorname{GRR}\left(\epsilon_x, x_i\right) & \mbox{with probability } \beta \\
                \operatorname{Unif}(1,\gamma_i) & \mbox{with probability } 1-\beta, 
    \end{array}\right.
    \end{align}
    where $\operatorname{GRR}(\cdot)$ denotes the response obtained by applying randomized response and $\operatorname{Unif}(\cdot)$ denotes a sample drawn from a discrete uniform distribution. For each of the $m$ sampled features, the sampling and reporting satisfy $\epsilon_x$-LDP, and for each of the $d-m$ features, the reporting ensures total randomness. Since $\epsilon$-LDP composes~\cite{DBLP:journals/fttcs/DworkR14}, the mechanism composes as $m\cdot\epsilon_x$. Additionally, since we incorporate sampling at the rate of $\beta$ for each user, we are able to leverage privacy amplification enabled by the combination of sampling with a randomization mechanism as shown in \cref{lem:amplify}. Therefore, for GRR-FS with sampling rate $\beta$ and applied privacy budget $\epsilon_x$ for each of $m$ sampled features, the mechanism satisfies LDP with \mbox{$\epsilon_X = \ln{\left(1 + \frac{m}{d}(e^{m\epsilon_x}-1)\right)}$}.
\end{proof}

As data randomized via an LDP protocol is invariant to post-processing~\cite{DBLP:journals/fttcs/DworkR14}, the server can utilize the collected data without compromising user privacy. Furthermore, as the server queries each user for their feature values only once, privacy degradation via repeated queries is also avoided. Also, GRR-FS can be applied to continuous node features after discretizing them into categorical attributes. However, directly using noisy features for learning can significantly impact the model's capability of obtaining high-quality embeddings and generalizing to new or unseen data. In this paper, we propose to utilize statistical frequency estimation techniques to approximately reconstruct user features and labels and better facilitate model training.

We first derive frequency estimation for GRR-FS. On the server side, the aggregator collects responses from $n$ users but is unaware whether an individual response was derived from GRR or sampled from the uniform distribution. For an arbitrary feature $x_i$, the aggregator can compute the frequency estimate for its $j$-th value, denoted by $\Tilde{\pi}_j$, as follows:
\begin{equation}\label{eq:grr_fs_estimate}
   \Tilde{\pi}_j = \frac{\lambda'_jd}{m(p-q)} + \frac{m-d-m\gamma_iq}{m\gamma_i(p-q)},
\end{equation}
with estimated variance:
\begin{equation}\label{eq:grr_fs_variance}
    \hat{\operatorname{var}}(\Tilde{\pi}_j) = \frac{d^2\tilde{\lambda}_j(1-\Tilde{\lambda}_j)}{m^2(n-1)(p-q)^2},         
\end{equation}
where, $\Tilde{\lambda}_j = \frac{1}{d}\left(\Tilde{\pi}_jm\gamma_i(p-q) + m(q\gamma_i-1) + d\right)$.

\begin{theorem}
    For an arbitrary attribute $x_i$ for $i\in\{1,\hdots,d\}$, randomized using GRR-FS with sampling probability $m/d$, the estimation $\Tilde{\pi}_j$ in (\ref{eq:grr_fs_estimate}) is an unbiased estimation of the true proportion of users having the $j$-th value with estimated variance $\hat{\operatorname{var}}(\Tilde{\pi}_j)$ in (\ref{eq:grr_fs_variance}).
\end{theorem}

\begin{proof}
    For an arbitrary feature $x_i$, let $\pi_j$ denote the true proportion of users who have the $j$-th value for $j\in[1,\hdots,\gamma_i]$. Under GRR-FS, the probability of observing the $j$-th value is
    \begin{align}
        \lambda_j &= \frac{1}{n}\left[\frac{m}{d}(n\pi_jp + n(1-\pi_j)q_2) + \left(1-\frac{m}{d}\right)\frac{n}{\gamma_i}\right] \\
        &= \frac{1}{d\gamma_i}\left[m\pi_j\gamma_i(p-q) + mq\gamma_i + d - m\right] \label{eq:prob_of_j_value}.
    \end{align}
    An unbiased estimator of $\lambda_j$ is the observed proportion of the $j$-th value, $\lambda'_j=n_j/n$, where $n_j$ refers to the number of users who report the $j$-th value. Then from (\ref{eq:prob_of_j_value}), it follows
    \begin{align}\label{eq:estimate}
        \Tilde{\pi}_j = \frac{\lambda'_jd}{m(p-q)} + \frac{m-d-m\gamma_iq}{m\gamma_i(p-q)^2}.
    \end{align}
    From (\ref{eq:estimate}), we have
    \begin{align}
        \operatorname{var}(\Tilde{\pi}_j) = \frac{d^2\operatorname{var}(n_j)}{n^2m^2(p-q)^2}.
    \end{align}
    We observe that $n_j$ follows the binomial distribution with parameters $n$ and $\lambda$ such that $\operatorname{var}(n_j) = n\lambda_j(1-\lambda_j)$ and we get
    \begin{align}\label{eq:grr_fs_true_var}
        \operatorname{var}(\Tilde{\pi}_j) = \frac{d^2\lambda_j(1-\lambda_j)}{nm^2(p-q)^2}.
    \end{align}
    Since $\mathbb{E}(n_j) = n\lambda_j$ and $\mathbb{E}(n_j^2) = n\lambda_j + n(n-1)\lambda_j$, we have
    \begin{align}
        \mathbb{E}\left[\frac{\Tilde{\lambda}_j(1-\Tilde{\lambda}_j)}{n-1}\right] = \frac{\lambda_j(1-\lambda)}{n},
    \end{align}
    where we use $\Tilde{\lambda}_j$ to mean the estimated probability of observing the $j$-th value obtained as $\Tilde{\lambda}_j = \frac{1}{d}\left(\Tilde{\pi}_jm\gamma_i(p-q) + m(q\gamma_i-1) + d\right)$.
    Finally, an estimation of the variance in (\ref{eq:grr_fs_true_var}) is given as
    \begin{align}
        \hat{\operatorname{var}}(\Tilde{\pi}_j) = \frac{d^2\Tilde{\lambda}_j(1-\Tilde{\lambda}_j)}{m^2(n-1)(p-q)^2}.
    \end{align}

\end{proof}

\begin{algorithm}[htb]
\caption{Node Feature Reconstruction}\label{alg:f_X}
\KwIn{$\mathcal{G}=(\mathcal{V},\mathcal{E})$, perturbed features $\mathbf{X}'$, feature domain sizes $\gamma = \{\gamma_1,\hdots,\gamma_d\}$, number of features $d$, number of sampled features $m$, privacy budget $\epsilon$, number of hops $K$}
\KwOut{Reconstructed features $\mathbf{\Tilde{X}}$}

\SetKwProg{Fn}{Function}{:}{end}
\Fn{$f_X(\mathcal{G},\mathbf{X}',\gamma,\epsilon,d,m,K)$}{

\For{$i=1,\hdots,d$}{
    $\boldsymbol{\lambda}' \leftarrow \operatorname{one-hot}(\mathbf{X}'_i)$

    \For{$k=1,\hdots,K$}{
        \For{$v \in \mathcal{V}$}{
        $\boldsymbol{\lambda}'_{v} \leftarrow \operatorname{MEAN}(\{\boldsymbol{\lambda}'_{v}\} \cup \{\boldsymbol{\lambda}'_{u}, \forall u \in \mathcal{N}(v)\})$
        }
    }
    
    $p \leftarrow \frac{e^{\epsilon}}{e^{\epsilon}+\gamma_i-1}$ and $q \leftarrow \frac{1}{e^{\epsilon}+\gamma_i-1}$
    
    \For{$v \in \mathcal{V}$}{
        \For{$j=1,\hdots,\gamma_i$}{
                $\Tilde{\pi}_{vj} \leftarrow \frac{d{\lambda}'_{vj}}{m(p-q)} + \frac{m-d-m\gamma_iq}{m\gamma_i(p-q)}$
        }
    
        $\Tilde{x}_v \leftarrow \argmax\limits_j\Tilde{\pi}_{vj}$
    }
    
    $\mathbf{\Tilde{x}}_i \leftarrow \{\Tilde{x}_v, \forall v\in\mathcal{V}\}$
}

\Return $\mathbf{\Tilde{X}} = \{\mathbf{\Tilde{x}}_1,\hdots,\mathbf{\Tilde{x}}_d\}^\intercal$
}
\end{algorithm}

For an arbitrary feature $x_i$, (\ref{eq:grr_fs_estimate}) gives us an unbiased estimate of the proportion of nodes with the $j$-th value over the whole graph. Pointwise reconstruction, however, is still intractable. Nonetheless, we can estimate feature occurrences $\boldsymbol{\Tilde{\pi}}$ at a sub-graph level for a reasonable number of nodes. In this case, $n$ becomes the number of nodes in the sub-graph, and $\lambda'_j$ is the observed proportion of $x_i$'s $j$-th value in the sub-graph. We then extend this sub-graph level estimation to node neighborhoods as shown in \cref{alg:f_X}. For a node $v \in \mathcal{V}$, we first obtain $\boldsymbol{\lambda}'_v$, the proportional frequency of an arbitrary feature $x_i$ from its neighborhood. Using (\ref{eq:grr_fs_estimate}), we estimate $v$'s true neighborhood feature distribution as $\tilde{\pi}_{vj}$ for the $j$-th value in the domain of $x_i$. Furthermore, considering that nodes in a graph usually share similar characteristics with their neighbors due to homophily~\cite{doi:10.1146/annurev.soc.27.1.415}, we reason that the feature distribution of a node should be very close to the feature distribution in its neighborhood. Ultimately, we obtain the highest probable feature value from the reconstructed neighborhood feature distribution $\Tilde{\pi}_v$ and assign it as $v$'s new reconstructed feature $\Tilde{x}$. In a simpler case with binary feature values, we can assign \mbox{$\Tilde{\pi}_{v2}$ ($\Tilde{\pi}_{v1}$)} to be node $v$'s new reconstructed feature to incorporate the uncertainty of the estimates.

That said, node degrees in real-world graphs generally follow a power law distribution resulting in a high number of nodes having low degrees~\cite{doi:10.1126/science.286.5439.509}. This variation in node degrees affects the reconstruction process as the neighborhood size directly influences the variance in estimating the true feature distribution. In (\ref{eq:grr_fs_variance}), for a fixed privacy budget $\epsilon_x$, we obtain a fixed transition probability $p$; the variance is then inversely proportional to $n$ which refers to the neighborhood size. This results in low-degree nodes having a higher variance in their estimates which may lead to inaccurate reconstruction. To minimize this effect, we implement multi-hop feature aggregation to increase the neighborhood size for low-degree nodes before computing the estimate as shown in~\cref{alg:f_X}.

\subsection{Reconstruction with Private Labels}\label{sec:f_Y}
In this section, we discuss the privatization and reconstruction of node labels. We again implement Generalized Randomized Response to add class-independent and symmetric noise to the labels. Here, the probabilities $p$ and $q$ are set as $\frac{e^{\epsilon_y}}{e^{\epsilon_y}+c-1}$ and $\frac{1}{e^{\epsilon_y}+c-1}$ respectively where $c$ indicates the number of classes and $\epsilon_y$ is the allocated label privacy budget. The server then collects these perturbed label vectors $\mathbf{y}_v'$ from all labeled nodes and can compute graph-level aggregates from the collected data. We can further leverage frequency estimation to obtain unbiased estimates of these aggregates. Let $\boldsymbol{\pi} \in \mathbb{R}^c$ denote the label distribution vector over $n$ nodes and $\boldsymbol{\lambda}'$ denote the label distribution observed by the server containing the sample proportions corresponding to $\boldsymbol{\pi}$. Then an unbiased estimate of $\boldsymbol{\pi}$ is obtained as~\cite{DBLP:conf/edbt/0009WH16}
\begin{equation}\label{eq:poly_rr_pi}
    \boldsymbol{\Tilde{\pi}} = \mathbf{P}^{-1}\boldsymbol{\lambda}',
\end{equation}
where $\mathbf{P}$ refers to the label transition matrix with diagonal entries $p$ and non-diagonal entries $q$. The variance in estimating $\boldsymbol{\tilde{\pi}}$ is obtained from the diagonal elements of the  dispersion matrix
\mbox{$\operatorname{disp}(\boldsymbol{\tilde{\pi}}) = (n-1)^{-1}\mathbf{P}^{-1}(\boldsymbol{\lambda}'^\delta - \boldsymbol{\lambda}'\boldsymbol{\lambda}'^\intercal)(\mathbf{P}^\intercal)^{-1}$}
where $\boldsymbol{\lambda}'^\delta$ is a diagonal matrix with the same diagonal elements as those of $\boldsymbol{\lambda}'$ and $(\cdot)^\intercal$ indicates the transpose operation.

\begin{algorithm}[t]
\caption{Node Label Reconstruction}\label{alg:f_Y}
\KwIn{$\mathcal{G}=(\mathcal{V},\mathcal{E})$, perturbed labels $\mathbf{Y}'$, number of classes $c$, privacy budget $\epsilon$, number of hops $K$}
\KwOut{Reconstructed labels $\mathbf{\Tilde{y}}_v \:\:\forall v\in \mathcal{V}^L$}

\SetKwProg{Fn}{Function}{:}{end}
\Fn{$f_Y(\mathcal{G},\mathbf{Y}',\epsilon,c,K)$}{

$\boldsymbol{\lambda}'_v \leftarrow \mathbf{y}'_v \:\:\forall v \in \mathcal{V}^L$

$\boldsymbol{\lambda}'_v \leftarrow  \Vec{0} \quad \forall v \in \mathcal{V}^U$

\For{$k=1,\hdots,K$}{
    \For{$v \in \mathcal{V}$}{
    $\boldsymbol{\lambda}'_{v} \leftarrow \operatorname{MEAN}(\{\boldsymbol{\lambda}'_{v}\} \cup \{\boldsymbol{\lambda}'_{u}, \forall u \in \mathcal{N}(v)\})$
    }
}

$p \leftarrow \frac{e^{\epsilon}}{e^{\epsilon}+c-1}$ and $q \leftarrow \frac{1}{e^{\epsilon}+c-1}$

Construct transition matrix $\mathbf{P} \in \mathbb{R}^{c\times c}$ using $p$ and $q$

\For{$v \in \mathcal{V}^L$}{
    $\boldsymbol{\Tilde{\pi}}_{v} \leftarrow \mathbf{P}^{-1}\boldsymbol{\lambda}'_{v}$

    $\mathbf{\Tilde{y}}_v \leftarrow \textnormal{one-hot}(\boldsymbol{\Tilde{\pi}}_v)$
}
\Return $\mathbf{\Tilde{y}}_v \:\:\forall v\in \mathcal{V}^L$
}
\end{algorithm}

Similar to node features, we can obtain frequency estimates of labels at a sub-graph level using (\ref{eq:poly_rr_pi}). \cref{alg:f_Y} shows the mechanism for label reconstruction via frequency estimation. Unlike node features, node labels are only provided for a small subset of nodes in the training graph. Direct propagation over all nodes is not feasible in such graphs. However, unlabeled nodes are also important for message propagation during multi-hop aggregation. To this end, we perform masked propagation to obtain label frequencies in multi-hop node neighborhoods. We mask the unlabeled nodes by setting their label vectors as an all-0 vector of size $c$. We then use the obtained neighborhood-level label distribution $\boldsymbol{\lambda}'_v$ to estimate the true label distribution $\boldsymbol{\Tilde{\pi}}_v$ using (\ref{eq:poly_rr_pi}). Assuming homophily, we assign the reconstructed neighborhood label distribution as the node's new label $\mathbf{\Tilde{y}}_v$ after one-hot encoding such that $\mathbf{\Tilde{y}}_v \in \{0,1\}^c$.
With the reconstructed labels, we define the GNN objective as
$\mathcal{L}_{gnn} = \sum\limits_{v \in \mathcal{V}^L} \ell(\mathbf{\hat{y}}_v,\mathbf{\Tilde{y}}_v)$,
where $\ell(\cdot)$ is the cross-entropy loss, $\mathbf{\Tilde{y}}_v$  and $\mathbf{\hat{y}}_v$ denote the reconstructed and predicted label of node $v$, respectively. 

\begin{algorithm}[tbh]
\caption{Reconstruction based Private GNN (RGNN)}\label{alg:rgnn}
\KwIn{Graph $\mathcal{G} = (\mathcal{V}, \mathcal{E})$, features $\mathbf{X}$, labels $\mathbf{Y}$, feature privacy budget $\epsilon_x$, label privacy budget $\epsilon_y$, feature randomizer $\mathcal{M}_X$, label randomizer $\mathcal{M}_Y$, feature reconstructor $f_X$, label reconstructor $f_Y$, hops for feature reconstruction $K_X$, hops for label reconstruction $K_Y$, clustering algorithm $h_\mathcal{E}$, number of clusters $C$, GNN model $h_\mathcal{G}$, regularization parameter $\alpha$}

\tcp{\textcolor{blue}{randomization}}
$\mathbf{x}_v' \leftarrow \mathcal{M}_X(\mathbf{x}_v; \epsilon_x) \:\: \forall v \in \mathcal{V}$

$\mathbf{y}_v' \leftarrow \mathcal{M}_Y(\mathbf{y}_v; \epsilon_y) \:\: \forall v \in \mathcal{V}^L$

\tcp{\textcolor{blue}{clustering}}
$\{\mathcal{C}_r\}_{r=1}^{C} \leftarrow h_{\mathcal{E}}(\mathcal{G})$ such that $\mathcal{C}_r \bigcap \mathcal{C}_s = \emptyset$ 

$\mathbf{b}_r' \leftarrow \frac{1}{|\mathcal{C}_r^L|}\sum\limits_{v\in\mathcal{C}_r^L}\mathbf{y}'_v \:\: \textnormal{for } r=\{1,..,C\}$

\tcp{\textcolor{blue}{reconstruction}}
$\mathbf{\Tilde{X}} \leftarrow f_X(\mathcal{G},\mathbf{X}',\gamma,\epsilon_x,d,m,K_x) $ 

$\mathbf{\Tilde{Y}}^L \leftarrow f_Y(\mathcal{G},\mathbf{Y}',\epsilon_y,c,K_y)$
    
obtain $\mathbf{\Tilde{b}}_r$ from $\mathbf{b}_r'\:\:\textnormal{for } r=\{1,\hdots,C\}$ using (\ref{eq:poly_rr_pi})

\tcp{\textcolor{blue}{GNN training}}
\For{$t=1,\hdots,T$}{

    $\mathbf{\hat{Y}} \leftarrow h_\mathcal{G}(\mathbf{\Tilde{X}},\mathcal{G};\mathbf{W})$ 

    $\mathbf{\hat{b}}_r \leftarrow \frac{1}{|\mathcal{C}_r^L
    |}\sum\limits_{v\in\mathcal{C}_r^L}\mathbf{\hat{y}}_v \: \textnormal{for } r=\{1,\hdots,C\}$ 

$\mathcal{L}_{gnn} \leftarrow \frac{1}{|\mathcal{V}^L|}\sum\limits_{v\in\mathcal{V}^L} \ell(\mathbf{\hat{y}}_v, \mathbf{\Tilde{y}}_v)$

$\mathcal{L}_{llp} \leftarrow \frac{1}{C} \sum_{r=1}^{C} D_{KL}(\mathbf{\hat{b}}_r||\mathbf{\Tilde{b}}_r)$

$\mathbf{W}^{t+1} \leftarrow \mathbf{W}^t - \eta \nabla \left(\mathcal{L}_{gnn} + \alpha\mathcal{L}_{llp} \right)$
}
\end{algorithm}

\begin{table*}[ht]
\scriptsize
    \centering
    \caption{Dataset statistics}
    \begin{tabular}{l>{\centering}m{1.2cm}>{\centering}m{1.2cm}>{\centering}m{1cm}>{\centering}m{1cm}>{\centering}m{1.2cm}>{\centering}m{0.6cm}>{\centering}m{0.6cm}>{\centering}m{1.2cm}>{\centering\arraybackslash}m{1.6cm}}
    \toprule
     Dataset &  $|\mathcal{V}|$ &  $|\mathcal{E}|$ & Avg. Deg. & $d$ &  Sparsity in $d$(\%) & $c$ & $d'$ & Sparsity in $d'$(\%) & Accuracy(\%) \\
    \midrule
    Citeseer &    3,327 &    4,552 &  2.7 & 3,703 & 99.2 & 6 & 53 & 56.0 & 74.7$\pm$1.2\\
        Cora &    2,708 &    5,278 &  3.9 & 1,433 & 98.7 & 7 & 58 & 73.8 & 87.5$\pm$0.9\\
        DBLP &   17,716 &   52,867 &  6.0 & 1,639 & 99.7 & 4 & 33 & 85.3 & 84.7$\pm$0.3\\
    Facebook &   22,470 &  170,912 & 15.2 & 4,714 & 99.7 & 4 & 48 & 82.8 & 94.2$\pm$0.3\\
    \hline
      German &      955 &    9,900 & 20.9 &    46 & 74.2 & 2 & -  &  -   & 88.1$\pm$4.0\\
     Student &      577 &    4,243 & 14.7 &    60 & 71.1 & 2 & -  &  -   & 86.3$\pm$3.8\\
    \bottomrule
    \end{tabular}
    \label{tab:datasets}
\end{table*}

\subsection{Reconstruction with Label Proportions}
The variance in node label estimation depends on neighborhood size which could be small for some nodes even after multi-hop aggregation. To alleviate this dependence on neighborhoods, we further introduce a reconstruction-based regularization constraint that incorporates the concept of learning from label proportions. This LLP-based regularization term enforces similarity constraints on bag-level label distributions instead of node-level label distributions. We introduce this LLP objective during training by creating bags that contain subsets of nodes in the graph. To this end, we utilize edge-based graph clustering algorithms and partition $\mathcal{G}$ into $C$ disjoint clusters, $\mathcal{C}_1,\hdots,\mathcal{C}_C$. In this work, we use METIS~\cite{DBLP:journals/siamsc/KarypisK98} to partition the graph. METIS is a multilevel $C$-way graph partitioning scheme that partitions $\mathcal{G}$ containing $N$ nodes into $C$ disjoint clusters such that each cluster contains around $N/C$ nodes and the number of inter-cluster edges is minimized. In the presence of homophily, nodes in the same cluster most likely share similar labels due to their proximity in $\mathcal{G}$. LLP on bags containing similarly labeled nodes approximates the standard supervised learning and is preferable for training~\cite{yu2015learning}. 

For each bag/cluster $\mathcal{C}_i$, we compute its label proportion as the mean of the label distribution of nodes in that bag/cluster. However, due to privacy constraints, the server cannot access true labels $\mathbf{y}_v$. So, we use randomized labels $\mathbf{y}'_v$ to obtain the perturbed bag proportions $\mathbf{b}_i'$ for each $\mathcal{C}_i$. Note that we only use the nodes that are also in $\mathcal{V}^L$ (denoted as $\mathcal{C}_i^L$ in Line 10 of \cref{alg:rgnn}) to obtain label proportions of bag $\mathcal{C}_i$. Since $\mathbf{b}_i'$ is an observed estimate of label proportion at a sub-graph level, we can obtain an unbiased estimate of the bag label proportion using frequency estimation. 
We then formulate the LLP-based learning objective as \mbox{$\min\limits_{\mathbf{W}} \mathcal{L}_{llp} = \sum_{i=1}^{C}D_{KL}(\mathbf{\hat{b}}_i||\mathbf{\Tilde{b}}_i)$}, where $\mathbf{\hat{b}}_i$ is the predicted label proportion of bag $i$ obtained by aggregating the predicted labels $\mathbf{\hat{y}}_v$ of nodes in $\mathcal{C}_i^L$, $\mathbf{\Tilde{b}}_i$ is the reconstructed label proportion of bag $i$ obtained using (\ref{eq:poly_rr_pi}). Also, $\mathbf{W}$ denotes the learnable parameters of a GNN model denoted as $h_{\mathcal{G}}(\cdot)$.
The overall GNN objective with LLP regularization follows as
\begin{equation}
    \min\limits_{\mathbf{W}} \mathcal{L}_{gnn} + \alpha \mathcal{L}_{llp}, 
\end{equation}
where $\alpha$ controls the influence of the LLP-based regularization. The overall training procedure of our model is presented in \cref{alg:rgnn}. On the user side, randomization is performed on both features and labels (Lines 1-2). The server has access to graph adjacency and performs a clustering operation on it to obtain clusters to be used as bags for LLP (Lines 3-4). The server uses the perturbed data along with the graph adjacency to estimate node features and labels, and bag proportions via reconstruction as discussed prior (Lines 5-7). Finally, the server trains a GNN model to fit the reconstructed features and labels with the LLP loss as regularization (Lines 9-13).

\begin{theorem}\label{thrm:GRR-FS}
\cref{alg:rgnn} satisfies $(\epsilon_X + \epsilon_y)$-LDP where $\epsilon_X = \ln{\left(1 + \frac{m}{d}(e^{m\epsilon_x}-1)\right)}$.
\end{theorem}
\begin{proof}
    We show in \cref{thm:GRR-FS_privacy_budget} that GRR-FS provides $\epsilon_X$-LDP. Also, the label randomization mechanism $\mathcal{M}_Y$ which implements GRR ensures $\epsilon_y$-LDP. According to the compositional property of DP~\cite{DBLP:journals/fttcs/DworkR14}, the randomized feature and label data collectively provide $(\epsilon_X+\epsilon_y)$-LDP. The server collects query responses from each node independently only once, thus maintaining the privacy guarantee. Finally, due to the invariance of randomization mechanisms to post-processing~\cite{DBLP:journals/fttcs/DworkR14}, privacy protection is not compromised during any step in the reconstruction or training process. Therefore, \cref{alg:rgnn} satisfies $(\epsilon_X+\epsilon_y)$-LDP.
\end{proof}

\section{Experiments}
\subsection{Experimental Settings}
\subsubsection{Datasets}\label{sec:datasets}

We perform evaluations on four real-world benchmark datasets: Citeseer~\cite{DBLP:journals/aim/SenNBGGE08}, Cora~\cite{DBLP:journals/aim/SenNBGGE08} and DBLP~\cite{DBLP:conf/iclr/BojchevskiG18} are well-known citation datasets where nodes represent papers and edges denote citations. Each node is described by bag-of-words features and a label denoting its category. Facebook~\cite{musae} is a social network dataset with verified Facebook sites as nodes and mutual likes as links. Node features represent site descriptions and label indicates its category. Statistics of the datasets are presented in~\cref{tab:datasets}. Sparsity in $d$ highlights the imbalanced distribution of binary feature values as we discussed in \cref{sec:intro}. To reduce such feature matrix sparsity, we preprocess these real-world datasets by combining a fixed number of features (70, 25, 50, and 100 for Citeseer, Cora, DBLP, and Facebook respectively) into one representative feature resulting in a lower feature dimension indicated as $d'$ in \cref{tab:datasets}. This process only requires the server to communicate the number of features to be grouped together to the users and does not affect the privacy guarantees of RGNN. 

We also evaluate RGNN on two semi-synthetic datasets consisting of edges constructed based on ranked feature similarity. In German~\cite{10020610} dataset, nodes represent clients at a German bank and the label classifies clients as good or bad customers. In Student~\cite{10020610} dataset, nodes represent students at two Portuguese schools and the label indicates their final grade. Both of these datasets contain both numerical and categorical features and we discretize the numerical attributes based on the quantile distribution of their values. Compared to real-world datasets, both German and Student are smaller in scale with relatively balanced distribution in the feature domains. 

\begin{table*}[ht]
    \scriptsize
     \centering
     \caption{Accuracy of RGNN and baselines under $(\epsilon_X+\epsilon_y)$-LDP with $m=10$}
     \begin{tabular}{>{\centering}m{1.5cm}|>{\centering}m{1.2cm}|>{\centering}m{2cm}|c|c|c|c}
     \hline
     Dataset & $\epsilon_x$ & Model & $\epsilon_y=3$ & $\epsilon_y=2$ & $\epsilon_y=1$ & $\epsilon_y=0.5$ \\
     \hline
     \multirow{6}*{Citeseer} & \multirow{3}{*}{1} & GraphSAGE & 24.4 $\pm$ 1.3 & 24.3 $\pm$ 2.6 & 20.5 $\pm$ 1.2 & 19.4 $\pm$ 0.9 \\ 
                                            &    &     LPGNN & 45.6 $\pm$ 12.2 & 51.7 $\pm$ 1.8 & 46.4 $\pm$ 1.6 & 26.9 $\pm$ 5.9 \\
                                            &    &      RGNN & \textbf{55.7 $\pm$ 1.4} & \textbf{52.6 $\pm$ 2.7} & \textbf{47.1 $\pm$ 2.8} & \textbf{36.2 $\pm$ 4.7} \\
    \cline{2-7}
                        & \multirow{3}{*}{0.1}   & GraphSAGE & 25.2 $\pm$ 1.3 & 24.6 $\pm$ 1.1 & 19.4 $\pm$ 1.2 & 19.4 $\pm$ 1.2 \\ 
                                            &    &     LPGNN & \textbf{53.1 $\pm$ 1.9} & \textbf{51.3 $\pm$ 3.4} & \textbf{47.3 $\pm$ 1.8} & 31.1 $\pm$ 6.3 \\ 
                                            &    &      RGNN & 51.7 $\pm$ 2.1 & \textbf{51.3 $\pm$ 2.3} & 45.1 $\pm$ 3.2 & \textbf{34.6 $\pm$ 3.0} \\ 
    \hline
    \multirow{6}*{Cora} & \multirow{3}{*}{1}     & GraphSAGE & 31.5 $\pm$ 1.9 & 28.0 $\pm$ 1.6 & 25.9 $\pm$ 2.4 & 20.4 $\pm$ 4.4 \\ 
                                            &    &     LPGNN & 55.5 $\pm$ 15.4 & 39.5 $\pm$ 14.2 & 34.1 $\pm$ 8.4 & 37.8 $\pm$ 14.2 \\
                                            &    &      RGNN & \textbf{77.8 $\pm$ 2.0} & \textbf{75.5 $\pm$ 1.6} & \textbf{67.5 $\pm$ 3.8} & \textbf{41.9 $\pm$ 3.0} \\ 
    \cline{2-7}
                        & \multirow{3}{*}{0.1}   & GraphSAGE & 32.2 $\pm$ 1.9 & 28.5 $\pm$ 0.8 & 24.8 $\pm$ 2.4 & 21.0 $\pm$ 3.3 \\ 
                                            &    &     LPGNN & 68.5 $\pm$ 1.8 & 63.8 $\pm$ 5.4 & 60.6 $\pm$ 3.5 & \textbf{44.8 $\pm$ 13.2} \\ 
                                            &    &      RGNN & \textbf{77.0 $\pm$ 1.5} & \textbf{75.8 $\pm$ 2.0} & \textbf{66.6 $\pm$ 5.0} & 40.6 $\pm$ 2.3 \\ 
    \hline
    \multirow{6}*{DBLP} & \multirow{3}{*}{1}     & GraphSAGE & 50.3 $\pm$ 1.4 & 50.2 $\pm$ 1.5 & 45.5 $\pm$ 0.9 & 42.4 $\pm$ 2.4 \\ 
                                            &    &     LPGNN & 65.2 $\pm$ 0.5 & 65.0 $\pm$ 0.4 & 60.7 $\pm$ 6.7 & 46.6 $\pm$ 1.5 \\ 
                                            &    &      RGNN & \textbf{71.2 $\pm$ 0.4} & \textbf{70.7 $\pm$ 0.7} & \textbf{70.0 $\pm$ 1.0} & \textbf{65.5 $\pm$ 2.3} \\ 
    \cline{2-7}
                        & \multirow{3}{*}{0.1}   & GraphSAGE & 49.1 $\pm$ 1.7 & 48.7 $\pm$ 2.4 & 44.5 $\pm$ 1.7 & 43.8 $\pm$ 1.4 \\ 
                                            &    &     LPGNN & 70.2 $\pm$ 1.1 & 68.5 $\pm$ 1.8 & 64.6 $\pm$ 4.4 & 59.9 $\pm$ 3.0 \\ 
                                            &    &      RGNN & \textbf{71.4 $\pm$ 0.7} & \textbf{71.0 $\pm$ 0.8} & \textbf{69.9 $\pm$ 0.9} & \textbf{65.8 $\pm$ 2.0} \\ 
    \hline
    \multirow{6}*{Facebook} & \multirow{3}{*}{1} & GraphSAGE & 46.4 $\pm$ 1.3 & 45.2 $\pm$ 1.3 & 35.5 $\pm$ 1.1 & 29.7 $\pm$ 1.4 \\ 
                                            &    &     LPGNN & 70.4 $\pm$ 2.1 & 69.3 $\pm$ 1.7 & 67.6 $\pm$ 1.9 & 58.7 $\pm$ 10.8 \\ 
                                            &    &      RGNN & \textbf{76.3 $\pm$ 0.9} & \textbf{75.9 $\pm$ 1.1} & \textbf{72.6 $\pm$ 1.6} & \textbf{64.5 $\pm$ 1.6} \\ 
    \cline{2-7}
                        & \multirow{3}{*}{0.1}   & GraphSAGE & 42.0 $\pm$ 1.3 & 41.4 $\pm$ 1.5 & 33.5 $\pm$ 1.2 & 29.6 $\pm$ 1.3 \\ 
                                            &    &     LPGNN & 76.3 $\pm$ 0.9 & 75.8 $\pm$ 0.8 & 73.1 $\pm$ 0.3 & \textbf{67.4 $\pm$ 1.3} \\ 
                                            &    &      RGNN & \textbf{76.5 $\pm$ 0.8} & \textbf{76.1 $\pm$ 1.2} & \textbf{73.2 $\pm$ 1.2} & 63.4 $\pm$ 1.1 \\ 
    \hline
    \multirow{6}*{German} &  \multirow{3}*{1}    & GraphSAGE & 76.8 $\pm$ 3.1 & 73.1 $\pm$ 4.3 & 70.9 $\pm$ 2.8 & 70.1 $\pm$ 4.2 \\ 
                                            &    &     LPGNN & 69.6 $\pm$ 1.5 & 69.6 $\pm$ 1.5 & 69.6 $\pm$ 1.5 & 69.6 $\pm$ 1.5 \\ 
                                            &    &      RGNN & \textbf{82.2 $\pm$ 3.5} & \textbf{82.3 $\pm$ 3.5} & \textbf{82.2 $\pm$ 6.2} & \textbf{83.1 $\pm$ 7.9} \\ 
    \cline{2-7}
                        &  \multirow{3}*{0.1}    & GraphSAGE & 74.1 $\pm$ 5.2 & 72.3 $\pm$ 2.9 & 70.8 $\pm$ 2.6 & 69.3 $\pm$ 6.3 \\ 
                                            &    &     LPGNN & 74.7 $\pm$ 5.0 & 73.1 $\pm$ 7.0 & 71.3 $\pm$ 6.7 & 72.8 $\pm$ 5. \\ 
                                            &    &      RGNN & \textbf{81.9 $\pm$ 4.4} & \textbf{81.9 $\pm$ 4.2} & \textbf{82.9 $\pm$ 6.9} & \textbf{83.2 $\pm$ 7.1} \\ 
    \hline
    \multirow{6}*{Student} & \multirow{3}*{1}    & GraphSAGE & 73.6 $\pm$ 4.5 & 68.9 $\pm$ 3.1 & 63.8 $\pm$ 7.8 & 56.2 $\pm$ 3.8 \\ 
                                            &    &     LPGNN & \textbf{90.3 $\pm$ 2.4} & \textbf{89.9 $\pm$ 2.0} & 85.1 $\pm$ 4.1 & 76.7 $\pm$ 9.7 \\ 
                                            &    &      RGNN & 88.9 $\pm$ 1.6 & 87.8 $\pm$ 1.2 & \textbf{88.2 $\pm$ 0.9} & \textbf{82.1 $\pm$ 3.6} \\ 
    \cline{2-7}
                        &  \multirow{3}*{0.1}    & GraphSAGE & 73.2 $\pm$ 4.7 & 68.8 $\pm$ 4.9 & 62.2 $\pm$ 6.7 & 56.9 $\pm$ 3.6 \\ 
                                            &    &     LPGNN & 88.1 $\pm$ 1.9 & \textbf{89.7 $\pm$ 1.9} & \textbf{87.4 $\pm$ 2.2} & 78.1 $\pm$ 9.7 \\ 
                                            &    &      RGNN & \textbf{88.5 $\pm$ 1.7} & 88.5 $\pm$ 1.1 & \textbf{87.4 $\pm$ 1.5} & \textbf{81.7 $\pm$ 4.2} \\ 
    \hline
    \end{tabular}
    \label{tab:compare_GRR-FS}
\end{table*}

\subsubsection{Baselines} 
For all datasets, we compare RGNN against the representative GNN model GraphSAGE. We train GraphSAGE directly on the perturbed features and labels obtained after randomization. Furthermore, we demonstrate RGNN's efficiency by comparing it against LPGNN~\cite{DBLP:conf/ccs/SajadmaneshG21} which performs node classification in the same node privacy scenario as described in this paper. For LPGNN, we use the multi-bit mechanism to perturb features and RR to perturb labels as proposed in the original paper. 

\subsubsection{Experimental Setup} 
We implement a two-layer GNN model with 16 hidden dimensions in our experiments. We use GraphSAGE as the base GNN model unless stated otherwise. We also implement GCN and GAT with four attention heads. For all models, we use ReLU~\cite{agarap2019deep} as the non-linear activation followed by dropout and train with Adam~\cite{DBLP:journals/corr/KingmaB14} optimizer. We vary the propagation parameters $K_x$ and $K_y$ among \{2, 4, 8, 16\}. For a fair comparison, we use the same propagation parameters in both RGNN and LPGNN for each dataset. We further vary the hyperparameters $C$ among \{4, 8, 16, 32, 64, 128, 256\} and $\alpha$ among \{0.01, 0.1, 1, 10, 20\}. 
To study the performance of RGNN under varying privacy budgets, we vary $\epsilon_x$ within \{1, 0.1, 0.01\} and $\epsilon_y$ within \{3, 2, 1, 0.5\}. The total feature privacy budget provided by GRR-FS, $\epsilon_X$, is then computed as in \cref{thm:GRR-FS_privacy_budget}. For instance, for the Citeseer dataset the corresponding values for $\epsilon_X$ varies within \{8.3, 0.3, 0.02\}. For all datasets, we randomly split nodes into training, validation, and test sets with 50/25/25\% ratios, respectively. We report the average results with standard deviations of 5 runs trained for 100 epochs each for all experiments. All models are implemented using PyTorch-Geometric~\cite{fey2019fast} and are run on GPU Tesla V100 (32GB RAM). 
Our implementation is available at { \url{https://github.com/karuna-bhaila/RGNN}}.

\subsection{Experimental Results}
\subsubsection{Comparison with Baselines}
The rightmost column in \cref{tab:datasets} shows node classification accuracies with original features. We provide this result as the non-private baseline measure of performance on each dataset. For all datasets, we fix $m$=10 and also set LPGNN's sampling parameter to be $m$. We evaluate all models on the semi-synthetic datasets and the pre-processed real-world datasets with reduced feature sparsity and report the results in \cref{tab:compare_GRR-FS} under $(\epsilon_X+\epsilon_y)$-LDP. Since LPGNN also performs sampling before randomization, we reason that the total feature privacy for LPGNN is also amplified resulting in $\epsilon_X$ feature privacy. We observe that RGNN significantly improves classification performance with some utility trade-off for all datasets. In the case of the semi-synthetic Student dataset, RGNN can achieve better or similar accuracy compared to the non-private baseline for higher label privacy budgets. This can be attributed to its size and the higher degree of feature homophily present in the graph owing to the nature of its construction. Generally, RGNN also achieves significantly better accuracy compared to LPGNN. Out of 48 scenarios, RGNN outperforms LPGNN for 39 of them. The difference in accuracy is mostly minimal otherwise. As discussed previously, the multi-bit mechanism in LPGNN preserves the sparsity of the feature matrix when $m < d$. Furthermore, even after pre-processing to reduce dimensions, the datasets are not evenly distributed in terms of the feature domains. Compared with this, RGNN has comparable or improved performance despite providing rigorous privacy protection by randomizing every feature. These results imply that the reconstruction-based framework is effective in improving model performance under LDP constraints.  

\subsubsection{Choice of GNN}\label{sec:choice_gnn}
\begin{figure}
  \begin{minipage}[t]{.5\linewidth}
    \centering
    \includegraphics[width=\textwidth]{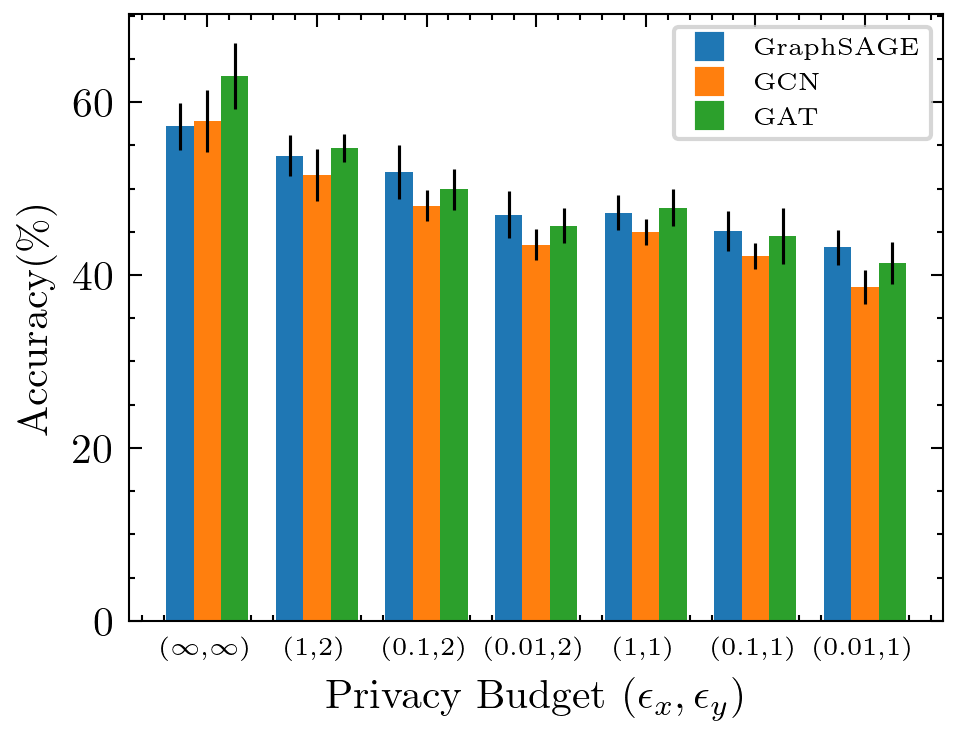}%
    \caption
      {Comparison of GNN architectures 
        \label{fig:choice_gnn}
      }
  \end{minipage}\hfill
  \begin{minipage}[t]{.47\linewidth}
    \centering
    \includegraphics[width=\textwidth]{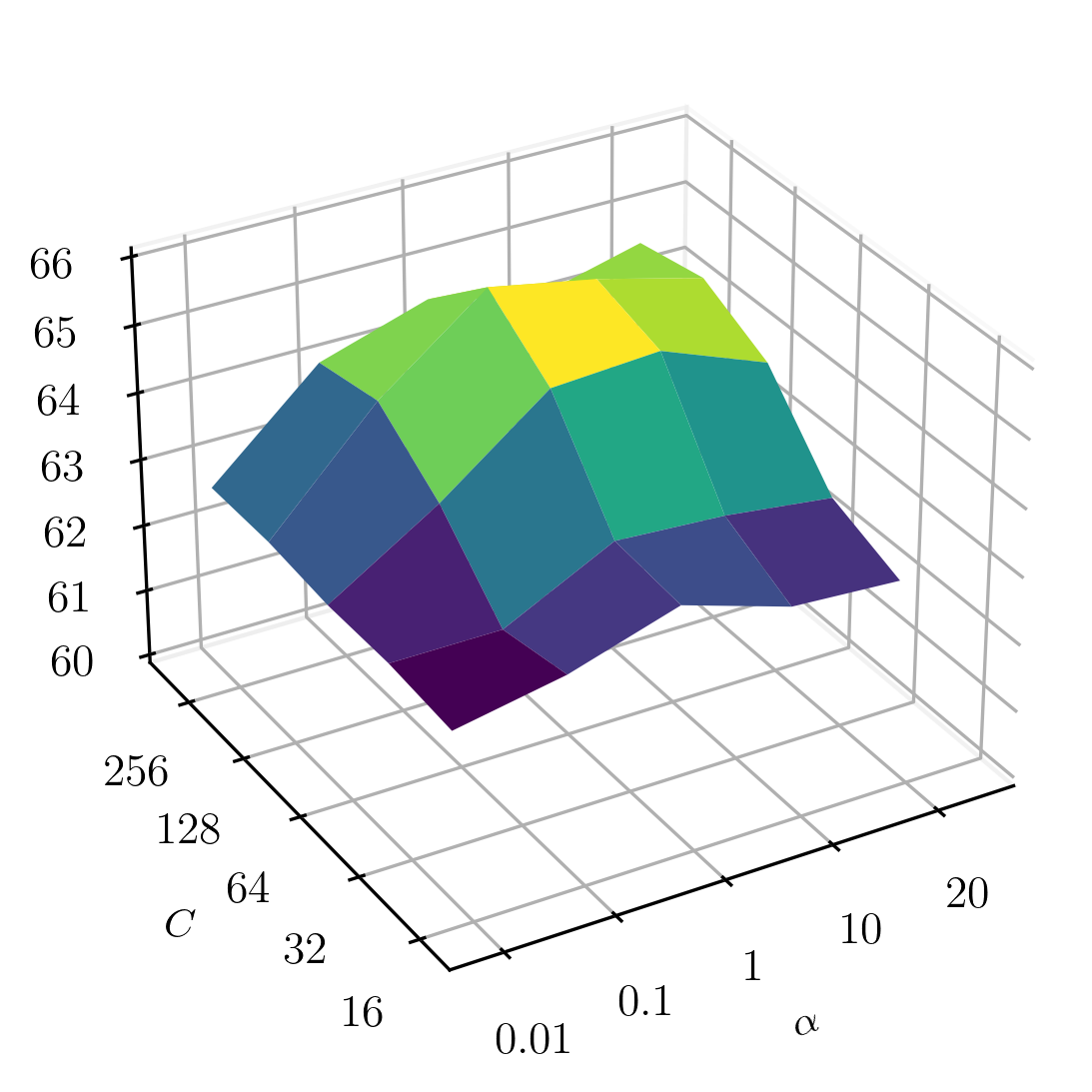}%
    \caption
      {Hyperparameter study of $\alpha$ and  $C$  \label{fig:alpha}}
      \end{minipage}
\end{figure}
In \cref{fig:choice_gnn}, we plot the performance of RGNN that uses GraphSAGE, GCN, and GAT as the backbone GNN. We only report the results on Citeseer as we observed similar trends on other datasets. In the non-private setting $(\infty,\infty)$, GAT outperforms both GCN and GraphSAGE. In the private scenario, GAT and GraphSAGE have relatively similar performances and GCN has comparatively lower utility. This trend suggests that the attention mechanism of GAT is able to effectively utilize the reconstructed features to compute attention coefficients. The neighborhood feature aggregator in GraphSAGE also benefits from such reconstruction. 
The results of this experiment demonstrate the flexibility of the proposed framework in fitting the reconstructed data to popular GNN models with satisfying performance.

\subsubsection{Ablation Study}\label{sec:ablation}
We conduct an ablation study to evaluate the contribution of the reconstruction and LLP components. To this end, we train three variants of RGNN: RGNN$\setminus f_X$ uses the perturbed features directly in training, RGNN$\setminus f_Y$ uses perturbed labels for training, and RGNN$\setminus$LLP without LLP regularization. With $m=10$, $\epsilon_x=1$ and $\epsilon_y=0.5$, we obtained the following results: (\textbf{40.2{\scriptsize$\pm$1.9}}, 34.0{\scriptsize$\pm$3.2}, 28.3{\scriptsize$\pm$6.6}, 23.6{\scriptsize$\pm$1.5}) for Citeseer and (\textbf{66.2{\scriptsize$\pm$1.4}}, 63.0{\scriptsize$\pm$1.0}, 64.9{\scriptsize$\pm$1.0}, 46.1{\scriptsize$\pm$1.3}) for DBLP with RGNN, and RGNN$\setminus$LLP, RGNN$\setminus f_Y$, RGNN$\setminus f_X$ respectively.
We observe that the performance of RGNN$\setminus f_X$ is significantly worse than that of RGNN, which highlights the contribution of denoising node features via frequency estimation. The performance of RGNN is also better than that of RGNN$\setminus f_Y$ and RGNN$\setminus$LLP,  showing that label reconstruction and LLP-based regularization improves model performance. 


\subsubsection{Hyperparameter Sensitivity}
We also analyze the influence of the LLP-based regularization term $\mathcal{L}_{llp}$. For this purpose, we vary $\alpha$ as $\{0.01, 0.1, 1, 10, 20\}$ and number of clusters $C$ as $\{16, 32, 64, 128, 256\}$ and report the results in \cref{fig:alpha} for a fixed privacy budget at $\epsilon_x=1$ and $\epsilon_y=0.5$. Due to space limitations, we only report the results for the DBLP dataset. We observe that generally as $\alpha$ increases, the performance increases and then decreases with the best performance at $\alpha \in \{0.1, 1, 10\}$. Also, performance increases as $C$ increases and the best results are obtained at $C=128$ for this dataset. Note that the optimal value of $C$ may vary depending on the size of the dataset which determines the bag size during clustering. The results demonstrate that the learning framework can benefit from the additional supervision provided by reconstructed LLP at a sub-graph level. 

\subsubsection{Propagation Parameter Study}\label{sec:kx_ky}
\begin{figure}[htb]
\begin{subfigure}{0.22\textwidth}
 \centerline{\includegraphics[width=\textwidth]{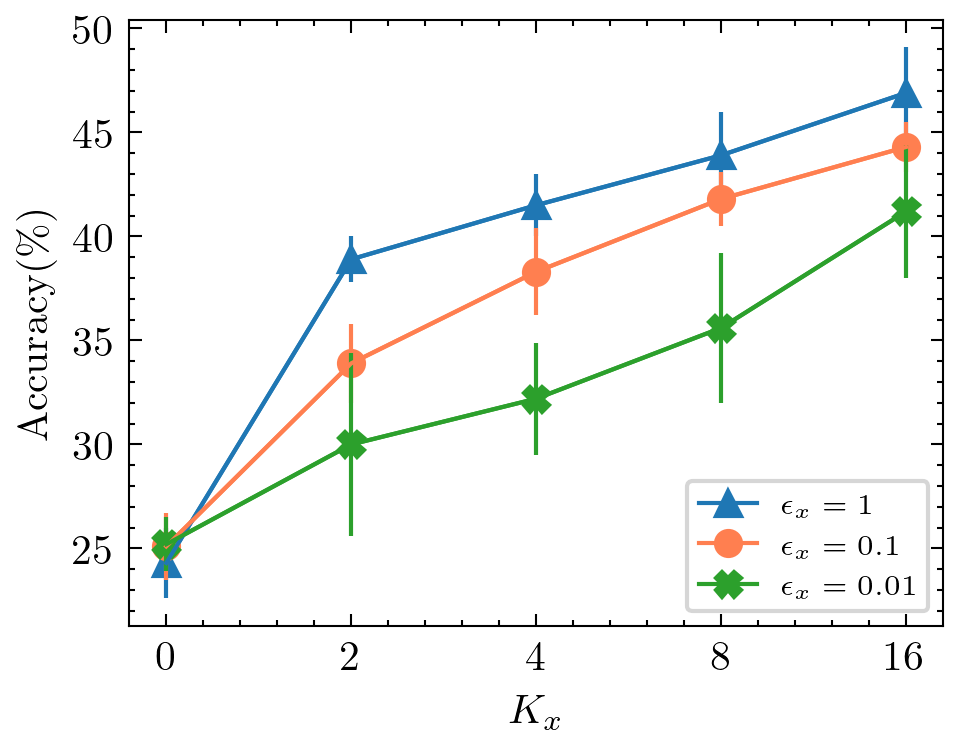}}
 \caption{$K_x$ on Citeseer}
\end{subfigure}
\hfill
\begin{subfigure}{0.22\textwidth}
 \centerline{\includegraphics[width=\textwidth]{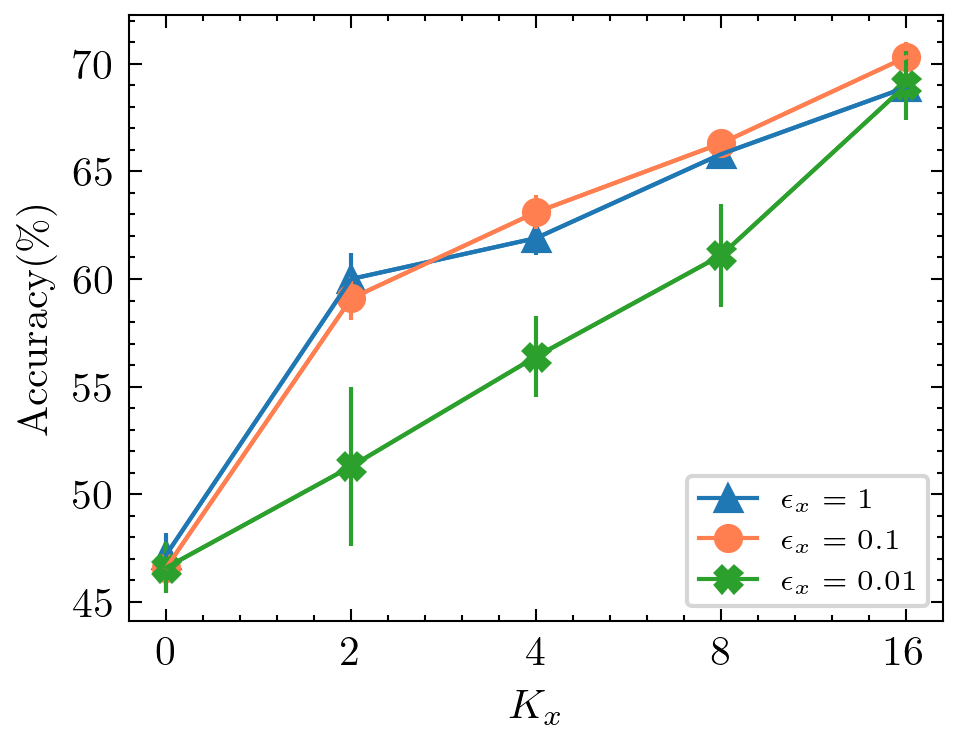}}
 \caption{$K_x$ on DBLP}
\end{subfigure}
\hfill
\begin{subfigure}{0.22\textwidth}
 \centerline{\includegraphics[width=\textwidth]{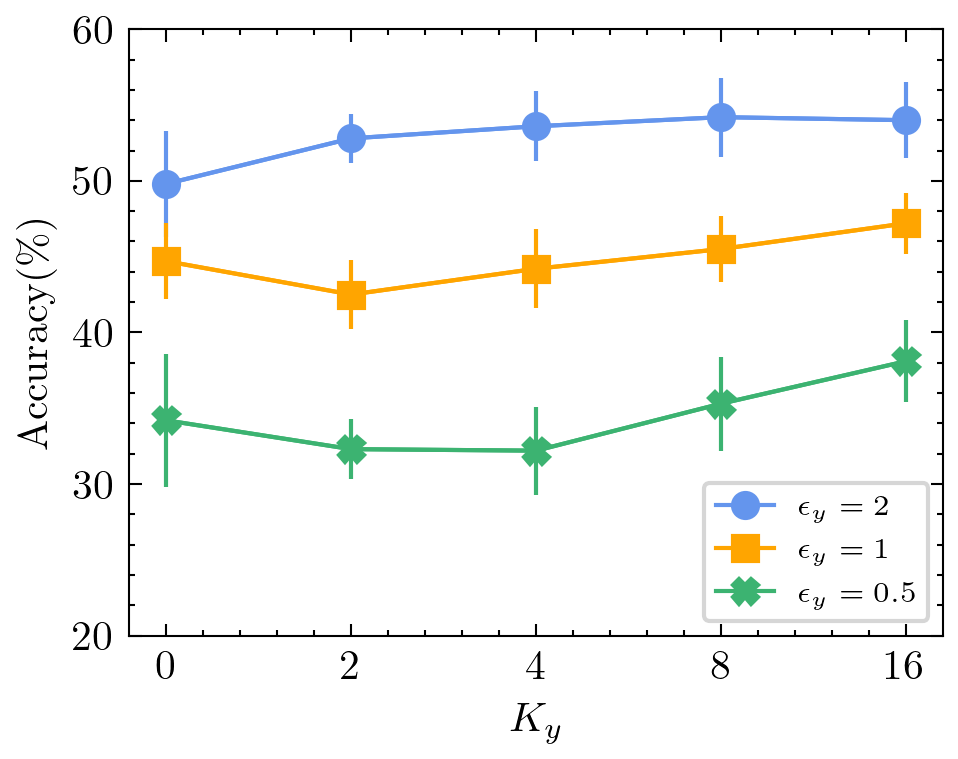}}
 \caption{$K_y$ on Citeseer}
\end{subfigure}
\hfill
\begin{subfigure}{0.22\textwidth}
 \centerline{\includegraphics[width=\textwidth]{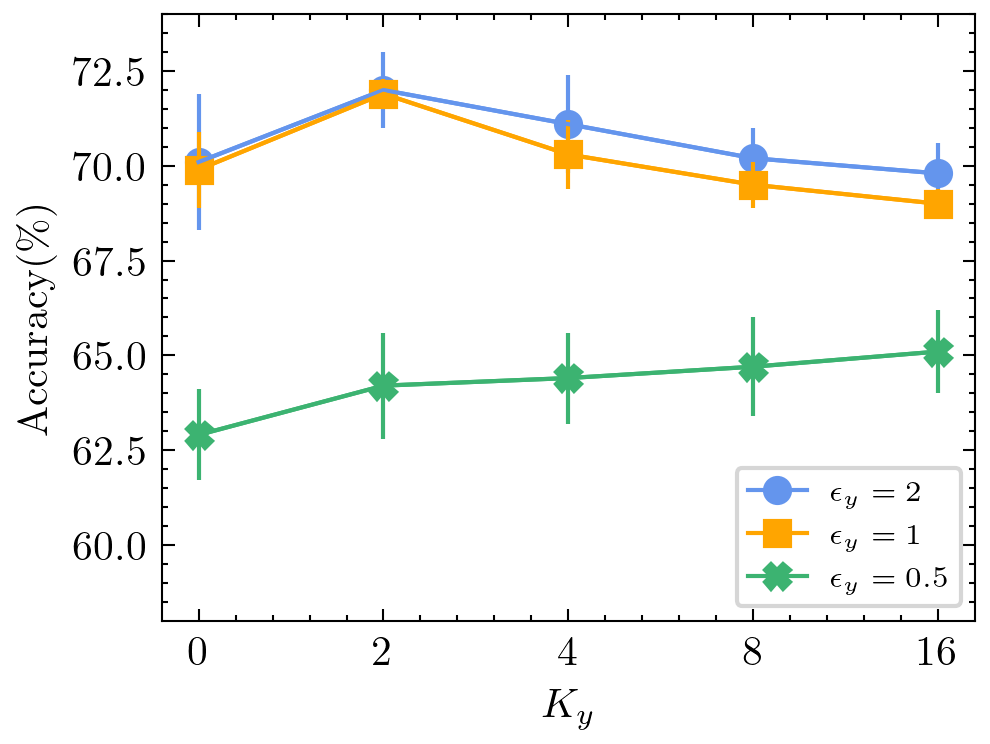}}
 \caption{$K_y$ on DBLP}
\end{subfigure}
\hfill
\caption{Influence of propagation parameters $K_x$ and $K_y$ on the performance of RGNN}
\label{fig:kx_ky}
\end{figure}
We further investigate the effect of propagation on the reconstruction framework. We vary the propagation parameters $K_x$ and $K_y$ among \{0, 2, 4, 8, 16\} and report the results in \cref{fig:kx_ky} for different values of $\epsilon_x$ and $\epsilon_y$ while fixing $m=10$. Due to space limit, we only report the results on Citeseer and DBLP. We observe that there is a drastic improvement in performance when increasing $K_x$ from 0 to 2 for both datasets. The performance further improves as $K_x$ increases for all values of $\epsilon_x$. This empirically proves that multi-hop propagation is effective in improving the estimates of node features during reconstruction.

Although the performance gain is not as drastic compared to $K_x$, model utility generally rises as $K_y$ increases. For Citeseer, at lower $\epsilon_y$, larger values of $K_y$ are needed to significantly improve performance. For DBLP, higher values of $K_y$ turn out to be detrimental to model performance for larger $\epsilon_y$ but beneficial for small $\epsilon_y$. We speculate that the comparatively higher node degrees in DBLP result in over-smoothing of the labels during multi-hop aggregation when lesser noise is added to them. 
This observation is also supported by our results on Facebook whose optimal $K_x$ \& $K_y$ both turn out to be 4. We generally conclude that RGNN can benefit from higher values of $K_x$ and $K_y$, especially with lower privacy budgets, but the selection of these parameters should depend on the desired privacy budget and average node degree of the graph.    

\section{Conclusion}
In this work, we present a reconstruction-based GNN learning framework in the node privacy setting, where features and labels are kept private from an untrusted aggregator. We first derive a general perturbation mechanism that incorporates sampling and is implemented at the user level to randomize multi-dimensional features. We then propose a flexible training framework that can significantly improve the privacy-utility tradeoff of the learner. The proposed method utilizes statistical frequency estimation to approximately reconstruct node features and labels. To reduce estimation variance, we incorporate a simple propagation mechanism that aggregates information from multi-hop neighborhoods before computing frequency estimates. We also introduce a regularization technique that uses label proportions of graph clusters to supervise the learning process at a sub-graph level. Our experiments demonstrated that our method generalizes well across various real-world datasets and GNN architectures while providing rigorous privacy guarantees compared to baseline methods.

\section*{Acknowledgements}
This work was supported in part by NSF 1946391 and 2119691.

\bibliographystyle{IEEEtran}  
\bibliography{references}  

\end{document}